\documentclass[sigconf,nonacm]{acmart}

\usepackage[T1]{fontenc}    
\usepackage{hyperref}       
\usepackage{url}            
\usepackage{booktabs}       
\usepackage{amsmath,amsthm, amsfonts}
\usepackage{nicefrac}       
\usepackage{microtype}      
\usepackage{lipsum}
\usepackage{listings}
\usepackage{graphicx}
\usepackage{xcolor}
\usepackage{setspace}
\definecolor{light-gray}{gray}{0.97}
\usepackage{algorithm,algpseudocode}

\newtheorem{theorem}{Theorem}

\DeclareMathOperator*{\argmax}{arg\,max}

\algdef{SE}[FUNCTION]{Function}{EndFunction}%
   [2]{\algorithmicfunction\ \textproc{#1}\ifthenelse{\equal{#2}{}}{}{(#2)}}%
   {\algorithmicend\ \algorithmicfunction}%

\AtBeginDocument{%
  \providecommand\BibTeX{{%
    \normalfont B\kern-0.5em{\scshape i\kern-0.25em b}\kern-0.8em\TeX}}}

\setcopyright{acmcopyright}
\copyrightyear{2020}
\acmYear{2020}
\acmDOI{xxx/xxx}

\begin{document}

\title[Interpretable Assessment of Fairness During Model Evaluation]{Interpretable Assessment of Fairness During Model Evaluation}

\newcommand{\liaff}{
		\affiliation{%
		\institution{LinkedIn Corporation}
		\streetaddress{1101 W Maude Ave} 
		\city{Sunnyvale, CA, USA}
		\postcode{94085}	
}}

\author{Amir Sepehri}
	\liaff 
	\email{asepehri@linkedin.com} 

\author{Cyrus DiCiccio}
\liaff 
\email{cdiciccio@linkedin.com}

\renewcommand{\shortauthors}{Sepehri and DiCiccio}

\begin{abstract}
  For companies developing products or algorithms, it is important to understand the potential effects not only globally, but also on sub-populations of users. In particular, it is important to detect if there are certain groups of users that are impacted differently compared to others with regard to business metrics or for whom a model treats unequally along fairness concerns. In this paper, we introduce a novel hierarchical clustering algorithm to detect heterogeneity among users in given sets of sub-populations with respect to any specified notion of group similarity.  We prove statistical guarantees about the output and provide interpretable results. We demonstrate the performance of the algorithm on real data from LinkedIn.
\end{abstract}



\maketitle

\section{Introduction}

For the vast majority of internet applications, changes to products or algorithms are evaluated through some combination of pre-launch exploratory metrics (which may come in the form of model evaluation metrics in the context of algorithms, or focus group research in the case of other product changes) as well as post-launch assessment of the impact seen from the change.    More concretely, the comprehensive evaluation of a new model typically consists of model evaluation metrics computed on the offline data available to develop a model and, pending the success of the model evaluation metrics, online A/B experimentation to understand the business impact of the new model.  In both of these stages, there is scope (and need) to evaluate the effectiveness and fairness of a model for subgroups of users.  This work proposes a clustering algorithm which provides a formal statistical test of the hypothesis of fairness with respect to some pre-specified notion accompanied by groupings of members demonstrating heterogeneity in the case of a rejection.  During the model training phase, there are several natural notions of fairness, such as simply comparing model performance across subgroups of members or comparing fairness metrics such as disparate impact or treatment.  During the A/B experimentation phase, comparing metric impact or lift across groups of members is a natural check for fairness.  In tackling the later problem, our algorithm is applicable to the detection of heterogeneous treatment effects, which is important in the fairness setting to ensure that the impact is beneficial to all users.  The algorithm developed in this work addresses the following three concerns regarding evaluating fairness:

First, for the algorithm to be useful, it needs to scale to internet industry sized datasets.  The computational complexity of our algorithm is $O(n + K^2)$ where n is the sample size and K is the number of subgroups of members under consideration. Most other methods are at least quadratic in the sample size. Others use variants of penalized regression and iterative algorithms where each step of the iteration is $O(K^2 n + K^3)$. We compare the computational complexity of our model to that of CausalTree, which is a state-of-the-art algorithm for estimation of heterogeneous treatment effects, and demonstrate that our algorithm is orders of magnitude faster. 

Second, there is a vast range of categorical features arising from a member base.  To name a few, there are demographic features, including gender, race, and country as well as market segment features such as device type or user behavior features.  Evaluating fairness according to the intersections of these features can be a daunting task.  Of course, pairwise comparisons can be made according to traditional fairness metrics; however, beyond the difficulty of correcting for multiple comparisons, the results will be nearly impossible to interpret.  The categorical features listed above easily define over a thousand sub-populations of members, and a list of pairwise comparisons of fairness would be difficult to digest and hardly actionable.   Our clustering algorithm produces groupings of members according to such categorical features.  The result is readily interpretable, and often actionable.  Once coarser groupings of members are formed, mitigation techniques such as segmented models or weighted training can be applied to ensure that fairness is attained for each of the groups.  These strategies are bolstered by first identifying coarser groupings of members, as the coarser groupings are more data rich than the original group.

Third, during the A/B experimentation phase, tests of statistical significance for metric impact (i.e. reporting p-values demonstrating the strength of evidence that the impact is significant) is the de-facto assessment tool for the success of a new model/etc in most internet applications.  Models will often not see a wide-spread launch without first demonstrating statistically significant, positive impact on the business metrics of interest.  Because rigorous statistical conclusions regarding the efficacy of a product/model launch, conclusions regarding the fairness and performance of a product for subgroups of users need also be held to this standard, and our algorithm provides a scalable mechanism for evaluating fairness along these standards. 

Motivated by this standard of evidence, we introduce a clustering algorithm providing inferential guarantees which easily scales to address the needs of internet applications. Our algorithm identifies groupings (or clusters) of members for whom the model performance is differing. The current work presents a novel method to detect heterogeneity in impact of a new product between predetermined groups of members.  This can be used to detect groupings of market segments that are affected similarly within grouping but differently between groupings.

The algorithm takes a notion of similarity (chosen to be pertinent to the problem at hand), as well as a partition of members (which can be chosen though demographic information such as gender, race, and country), and returns coarser, intersectional groupings of these members which the model performance is statistically significantly dissimilar.  Notions of similarity pertaining to fairness metrics as well as conventional A/B experimentation metrics are discussed in detail. We describe how this algorithm can be used during both the model development phase as well as during A/B experimentation to assess the fairness of a new model.  

More specifically, we provide an ``agglomerative clustering'' algorithm which begins with groupings (or clusters of users) and iteratively merges clusters according to their similarity in a predetermined sense. The idea is to not only determine whether all groupings of members are impacted similarly, but to identify further groupings of the members which demonstrate similar performance.  

Agglomerative hierarchical clustering works by starting with clusters consisting of the individual data points and, at each step, merging the two most similar clusters. It is customary to use a measure of dissimilarity between individual data points, and expand it to a notion of dissimilarity between two groups of data points using a linkage function. In this work, we introduce a novel notion of dissimilarity between groups of data points that incorporates the structure of the problem in hand and the statistical nature of the problem. The result is a hierarchical clustering algorithm that provides statistical guarantees and inference about the output.

The remainder of the paper is organized as follows.  Section \ref{sec:lit} provides a literature review situating the present work in the landscape of statistical testing, clustering, and detection of heterogeneity of treatment effects.  In Section \ref{sec:methodology}, we propose a novel clustering algorithm as well as a definition of similarity between clusters which is required to make rigorous statistical claims.  Formal details for evaluating the performance of a new model according to the pre- and post-launch criteria introduced in Section \ref{sec:criteria} using the methodology developed in Section \ref{sec:methodology} are provided in Section \ref{sec:similarity}.  Finally, the favorable performance of our proposed methodology is demonstrated through simulated data in Section \ref{sec:simulation} and through ''real-world'' data in Section \ref{sec:experiment}.

\section{Literature Survey}\label{sec:lit}

The clustering algorithm introduced in the present work provides a mechanism for evaluating fairness of a model during the model development phase as well as during A/B experimentation across a pre-specified set of categorical user attributes.  In doing so, this work also introduces a novel test for the detection of heterogeneous treatment effects.  We now position this work in existing literature.  

Traditionally, in the statistics literature, tests of homogeneity of groups have been performed through chi-squared or ANOVA tests.  Such tests typically can identify that the groups are not homogeneous, but are unable to characterize which groups are similar or different, without making use of ad-hoc metrics such as the chi-squared components.  These metrics are largely uninterpretable, particularly when there are a large number of groups under consideration and therefore are not well suited to the problem of assessing fairness.  There are several more modern families of methods, such as tree and clustering algorithms which can be useful for providing interpretable groupings of users, although many of these methods fail to provide meaningful statistical guarantees.  

A number of clustering algorithms such as K-means/mediods \cite{macqueen1967some,kaufman2009finding}, Bayes Classifiers, and Gaussian Mixture Models, which can provide disparate groups of users.  These can be somewhat limited in that they do not begin from groupings of users and they require a predetermined idea of the number of clusters.  Later developments such as incorporating information criteria lessen the need to pre-specify the number of clusters, but are still unsatisfying for our problem as we do not know a priori the number of dissimilar groupings of members.  Algorithms incorporating some initial notion of clusters include agglomerative and divisive hierarchical clustering \cite{sabine2001cluster}.  These methods typically require stopping rules which translate in non-obvious ways to error rates among the identified clusters.

Our work seeks to bridge these approaches by adding statistical guarantees to clustering algorithms which can readily (and scalably) be applied to assessing fairness in A/B testing settings.  The problem of assessing fairness during A/B experimentation has been largely unaddressed in the literature, though recently, it has been tackled by \cite{saintjacques2020}.  This work compares an Atkinson index metric (which measures inequality of metrics between users) between two models.  There are two advantages of this work over clustering approaches: the fairness is assessed at a user level which can identify unfairness that might not be apparent from comparing groups and this methodology does not require the user to specify attributes, and consequently can identify unfairness driven by attributes that a user may not think to specify.  That being said, the main drawback of this approach is that the results lack interpretability.  Specifically, the Atkinson index relies on a user selected parameter and depending on the choice of parameter, the direction of the unfairness may change.  As a result, the user cannot concretely infer which model is more fair.  Furthermore, the Atkinson index is a concentration based metric which provides evidence that there is inequality, but which does not provide insight into which groups of members may be underserved.  This work emphasizes the detection of underserved users, which is essential for informing mitigation strategies.   

Outside of fairness and A/B experimentation contexts, there have been several other works providing clustering algorithms with statistical guarantees, including \cite{chung20,kimes17,lenenstien04,Liu08}  .  This line of works are largely aimed at addressing genetic data, and often rely on assumptions (such as normality of the observed data) which are unrealistic outside of biological data and that make them unsuitable for applications in fairness.

In developing our algorithm, we also introduce a novel algorithm for detecting heterogeneous treatment effects based on this new framework. Detection of heterogeneous treatment effects is important across different applications, including health care, A/B testing in the internet industry \cite{alexDeng,deng2016concise,kunzel2019metalearners,wager2018estimation}, social sciences \cite{xie2012estimating,green2012modeling,hastings2006preferences}, as well as in experiments used in logistics and resource allocation \cite{imai2011estimation,zhou2017estimating}.  A line of work focusing on the detection of heterogeneous treatment effects along a predefined set of member dimensions such as location, language, etc. includes using penalized fixed effects for each subgroup \cite{deng2016concise,feller2009beyond,dixon1991bayesian}.  Many of these methods only declare whether there is heterogeneity and leave it to the user to find interpretable patterns. Others provide clusters, and hence offer more interpretable insights, though they often choose the number of clusters in ad-hoc manners.  Regardless, none of these methods provide statistical interpretation and inference for the output.  Another line of work in this direction, which is more readily applicable to fairness problems, focuses on detecting the subgroups of users that are impacted different from each other using a combination of machine learning and causal inference methods \cite{athey16,wager2018estimation,kunzel2019metalearners,grimmer2017estimating}.  Many of these works, including the Causal Tree algorithm, are well suited for continuous valued features, but may not perform well in the context of categorical features, which are of primary concern for fairness applications.  We demonstrate these shortcomings in our simulation section and provide evidence that our method is more effective in these contexts.

\section{Criteria for Evaluating a Model}\label{sec:criteria}

In this section, we describe how clustering can be used to evaluate the fairness and efficacy of a model 

\subsection{Pre-Launch: Evaluating Fairness of Machine Learning Models}
From access to healthcare to personal finance decisions, people are becoming increasingly reliant on the decisions made by machine learning models.  With the pervasiveness of these models, it is of great importance to ensure that the models are fairly treating all users. As a first step towards mitigating potential bias, we must understand the users which a model may be under serving.  Clustering based on fairness metrics can help to identify groupings of users for whom the model is giving undesirable performance, which can in turn be useful in mitigating such bias.  

As a motivating example, consider a model that is biased based on user age.  Further, suppose that the model tends to perform well on younger members, less well on middle aged members, and poorly on older members, but that the member ages are discretized into bins (perhaps by 5 year groupings so that users can be put into ranges, 20-24, 25-29, etc.).  In this example, an agglomorative clustering algorithm seeks to recover the groupings of members for which the model performs differently.  This can be helpful for interpretability to identify appropriate ranges from the finer partitioning of members. 

With a large number of protected attributes, and non-obvious interactions between said attributes, clustering can be very helpful for identifying coarser and more readily interpretable grouping of users. 

\subsection{Post-Launch: A/B Testing}
The efficacy of a product launch can differ across groupings of members.  To fully understand the performance, we monitor the results of our experiments along different member segments and product dimensions.  As a concrete example consider the geographical location of a member. When building a new product, we need to ensure that product ramp decisions are made not only based on the global, aggregated impact, but also take into consideration the impact in priority markets. An experiment can impact members in different countries differently, and a significant impact in an important country may not show up at the global aggregate level. 

Consider launching a new version of a product based on the results of an A/B test in \(K\) countries. Let \((\% \Delta_i, SD_i)\) be the estimated lift and its standard error for country \(i\). We want to determine if there is heterogeneity in \(\{\% \Delta_i\}\). This problem can be formally represented in different ways. One approach is to interpret heterogeneity in terms of clustering of the lifts: Are there groups of countries that are performing different from each other. This is the formulation used in this work. We introduce a novel hierarchical clustering algorithm to solve this problem.

\section{Methodology}\label{sec:methodology}

\subsection{Testing similarity of groups}\label{subsec:method}

Suppose that users are segmented into $K$ groups $C_1,...,C_K$ (e.g. by countries or protected attributes).  We propose a procedure which iterative merges the two "most similar" clusters until there is strong evidence that the remaining clusters are dissimilar.  At this point, the procedure rejects the null hypothesis that the clusters are all similar, and reports the aggregated clusters which appear to be dissimilar.  Formally, we are interested in testing the null hypothesis 
\begin{equation}\label{eqn:GlobalNull}
H_0: C_i \text{ is similar to } C_j \text{ for all } i \neq j, 1 \leq i,j \leq K.     
\end{equation}

Our algorithm requires a formal definition of ``similarity.''  Some examples are defining clusters to be similar if they have equal lift or are equivalent according to a fairness metric.  We also require that the cluster similarity is merge invariant under the specified metric in the following sense.

\begin{definition}
A notion of cluster similarity is said to be {\it merge-invariant} if for any similar clusters $C_1$, $C_2$, $C_3$, we have that $C_1$ is similar to the cluster formed by merging $C_2$ and $C_3$. 
\end{definition}

A discussion of when the merge-invariance property holds (and fails to) is given in the Appendix.  We also require that we have a method to obtain a p-value for testing the pairwise hypotheses 
\[
H_{i,j}: C_i \text{ is similar to } C_j
\]
for each $i$ and $j$.  Notions of similarity as well as the corresponding tests are given in subsequent sections.  

Our iterative algorithm proceeds as follows.  Begin with clusters $C_1,...,C_K$.  Test the null hypothesis 
\begin{align*}
H_0^{(0)}:& \text{ There exists a pair of clusters } C_i \text{ and } C_j \text{ such that } \\
& C_i \text{ is similar to } C_j
\end{align*}
that is, we test the null hypothesis that there exist at least one pair of similar clusters at level $\alpha/K$.  To perform this test, we find p-values $\hat p_{i,j}$ for testing the hypotheses
\[
H_{i,j}^{(0)}: C_i \text{ is similar to } C_j
\]
and reject the null hypotheses if all of the $\hat p_{i,j}$ are below $\alpha/K$ (which provides strong evidence that all of the clusters are dissimilar).  If the test rejects, there is evidence that all the nodes are dissimilar, so we report all clusters.  If the test fails to reject, ``merge'' the two most similar clusters (i.e. clusters $C_i$ and $C_j$ with largest p-value $\hat p_{i,j}$), resulting in new clusters $C_1^{(1)},...,C_{K-1}^{(1)}$.  Suppose that at iteration $b$, no rejections have been made.  Then, test 
\begin{align*}
H_0^{(b)}: & \text{ There exists a pair of clusters } C_i^{(b)} \text{ and } C^{(b)}_j\\  & \text{ such that } 
 C_i^{(b)} \text{ is similar to } C_j^{(b)}
\end{align*}
at level $\alpha/K$ by finding p-values $\hat p^{(b)}_{i,j}$ for the hypotheses  
\[
H_{i,j}^{(b)}: C^{(b)}_i \text{ is similar to } C^{(b)}_j
\]
and rejecting if $\hat p_{i,j}$ are below $\alpha/K$.  If this null hypothesis is rejected, stop the procedure and report clusters $C_1^{(b)},...,C_{K-b}^{(b)}$ as disparate.  Otherwise, merge the two most similar clusters (i.e. those clusters with larges p-value).  This procedure is formalized in Algorithm \ref{algo:clustering}.  Methods for efficiently performing the pairwise tests are given in Section \ref{sec:LikelihoodRatioComputation}.

\begin{algorithm}[!h]
	\caption{Testing Similarity via Agglomerative Clustering}\label{algo:clustering}
	\begin{algorithmic}[1]
		\State \text{Input: } Clusters $C_1,...,C_K$ and level of significance $\alpha$
		\State Output : A test decision for the null hypothesis, $H_0$, that all clusters are similar and and in the case of rejection, dissimilar clusters
		\State Set $C^{(0)}_i = C_i$ for $1\leq i \leq K$ and $b = 0$
		\While{$H_0$ has not been rejected}
		\State Compute p-values $\hat p^{b}_{i,j}$ for testing that cluster $C^{(0)}_i$ is similar to $C^{(0)}_j$ for each $i$ and $j$.
		\If{$\hat p^{b}_{i,j} < \alpha/K$ for all $i$ and $j$}
		\State Reject $H_0$
		\Else 
		\State Merge the two most similar clusters (i.e. those with largest $\hat p^{b}_{i,j}$) resulting in new clusters 
		\State $C^{(b+1)}_1,...,C^{(b+1)}_{K-b-1}$
		\State Set $b = b+1$
		\EndIf
		\EndWhile
		
		\Return test decision and disparate clusters $C^{(b)}_1,...,C^{(b)}_{K-b}$ in the case of rejection
	\end{algorithmic}
\end{algorithm}

We now state a theorem establishing the validity of the testing procedure given in Algorithm \ref{algo:clustering}.

\begin{theorem}\label{thm:validity}
Assuming that the notion of similarity is merge invariant and the p-values $\hat p^b_{i,j}$ are independently of the previous decisions to merge clusters, Algorithm \ref{algo:clustering} controls the probability of falsely rejecting the null hypothesis specified in Equation (\ref{eqn:GlobalNull}) at level $\alpha$.  Furthermore, the probability that the test rejects the null hypothesis and reports clusters $C_1^{(b)},...,C_{K-b}^{(b)}$ as disparate, when in fact at least one pair of these clusters is similar is bounded by $\alpha$.
\end{theorem}

The independence of the p-values can be met using data-splitting as follows.  If the available data is split into $K$ disjoint sets, and each split is used to test one of the sequential hypotheses, then the p-values meet this condition.  Alternatively, the data can be split into two portions.  If this method is applied on the first portion of the data using an arbitrary threshold, and then the remainder of the data is used to test the pairwise dissimilarity of the resultant clusters, the statistical significance is maintained.  Also, in many experimental settings, the data is observed sequentially over time, e.g. metrics may be aggregated daily.  As new data is observed, it can then be used to test one of the sequential hypotheses.  Alternatively, because there are ultimately of order $K^2$ pairwise hypotheses considered, we suggest replacing $\alpha/K$ by $\alpha/K^2$ hypotheses in situations where this independence condition is not met.  A Bonferroni style corrections suggests that this is an appropriate threshold, and the validity of tests using this threshold is demonstrated in the simulations section.  

\section{Notions of Similarity and Formal Testing}\label{sec:similarity}
Algorithm \ref{algo:clustering} requires a notion of similarity as well as a method of performing a hypothesis test that two clusters of users are similar with respect to this notion. We derive a novel testing framework for assessing equality of lift as well as suggest some well-known tests in the context of fairness.  

\subsection{Analyzing A/B Results}\label{sec:abtest}
Consider the problem of evaluating A/B experiment results for $K$ subpopulations of members, $C_1, \ldots, C_K$. As described in section \ref{subsec:method}, we are interestied in testing the testing the following null hypothesis
\begin{equation*}
H_0: C_i \text{ is similar to } C_j \text{ for all } i \neq j, 1 \leq i,j \leq K.     
\end{equation*}
for appropriate notions of similarity.  
Furthermore, to proceed with our algorithm, we need to find appropriate p-values for testing
\[
H_{i,j}: C_i \text{ is similar to } C_j.
\]
We give the details of appropriate choice of similarity and testing procedure for detecting heterogeneous treatment effects, as well as heterogeneity in lift.  

\subsubsection{Detecting Heterogeneous Treatment Effects}\label{sec:hte}

For many experiments, practitioners want to understand the treatment 

Typically, an A/B experiment analyses the average treatment effect (ATE) which implicitly assumes a similar impact across members.  Identifying heterogeneous treatment effects provides an understanding of the treatment impact on relevant subpopulations of members.  

For each $C_i$, define the treatment effect as $\tau_i$.  This treatment effect can be estimated by $\hat \tau_i$ defined to be the difference of the per-member metric estimated on the members exposed to the treatment model, and the per member metric estimated on the members exposed to the control model.  In most cases, $\hat \tau_i$ is approximately normally distributed with mean $\tau_i$ and variance that is easily estimated.  By defining $H_{i,j}$ as
\[
H_{i,j}: \tau_i = \tau_j,
\]
applying Algorithm \ref{algo:clustering} yields groupings of members displaying heterogeneous treatment effects.  Note that the hypothesis $H_{i,j}$ are readily tested using the usual two-sample $t$-test.  In addition to identifying heterogeneity in treatment effects, it is also important to understand the heterogeneity in metric lift, which is addressed in the next section.  

\subsubsection{Detecting Equality of Lift Using the Likelihood-Ratio Link Function}\label{sec:notion-loglike}

Assume that we observe the percentage lift in a metric and corresponding standard deviation \( (\% \Delta_i, SD_i)\) for each $C_i$. We model $\% \Delta_i$ as a normally distributed random variable around a true treatment effect $\mu_i$. That is, $\%\Delta_i \sim N(\mu_i, SD_i^2)$. This is a reasonable model as $\% \Delta_i$ is a function of sample averages; those averages are approximately normally distributed because of central limit theorem and any differentiable function of them as also approximately normally distributed per the standard delta method argument. Then, the $H_{i,j}$ can be framed as
\[
H_{i,j}: \mu_i = \mu_j.
\]
This hypothesis can be tested using a number of standard tests. One such test is the generalized likelihood ratio test, which coincides with the t-test in this specific case. The generalized likelihood ratio test compares the maximum likelihood of the data under the null hypothesis to the same quantity under the unrestricted model. The unrestricted parameter space here is $(\mu_i, \mu_j) \in \mathbb{R}^2$. Let $L(\mu_i, \mu_j)$ denote the likelihood of the data under the model given by $(\mu_i, \mu_j)$. Then, the likelihood ratio statistic is
\begin{align*}
    LR = - 2 \log \frac{\max_{\mu_i = \mu_j} L(\mu_i, \mu_j)}{\max_{\mu_i , \mu_j} L(\mu_i, \mu_j)}.
\end{align*}
It simplifies to
\begin{align}\label{eqn:basicLRT}
    LR = \frac{(\%\Delta_i - \%\Delta_j)^2}{SD_i^2 + SD_j^2}.
\end{align}
We know that under the null $LR$ is a chi-squared random variable with one degree of freedom. Therefore, $p_{i,j} = 1 - \chi_1^2(LR)$ is a valid p-value for testing $H_{i,j}$.

A similar approach works at each iteration of our iterative algorithm. At iteration $b$, we need to test
\begin{align*}
H_0^{(b)}: & \text{ There exists a pair of clusters } C_i^{(b)} \text{ and } C^{(b)}_j\\  & \text{ such that } 
 C_i^{(b)} \text{ is similar to } C_j^{(b)},
\end{align*}
which requires testing 
\[
H_{i,j}^{(b)}: C^{(b)}_i \text{ is similar to } C^{(b)}_j,
\]
where $C^{(b)}_i$ and $C^{(b)}_j$ are two cluster of countries formed in previous iterations. This is accomplished as follows. Assume \(C^{(b)}_i = \{C_{i_1},...,C_{i_{m_i}}\}\) and \(C^{(b)}_j = \{C_{j_1},...,C_{j_{m_j}}\}\) . The data is modeled as 
\begin{align*}
\%\Delta_{i_s} \sim N(\mu^{(b)}_i, SD_{i_s}^2),\\
\%\Delta_{j_s} \sim N(\mu^{(b)}_j, SD_{j_s}^2),
\end{align*}
where $\mu^{(b)}_i$ and $\mu^{(b)}_j$ are the common true treatment effect for countries in $C^{(b)}_i$ and $C^{(b)}_j$, respectively.

Under this model, we have
\[
H_{i,j}^{(b)}: \mu^{(b)}_i = \mu^{(b)}_j.
\]
This hypothesis can be tested using likelihood ratio test. Again, let  $L^{(b)}(\mu^{(b)}_i ,  \mu^{(b)}_j)$ be the likelihood function of the data under $(\mu^{(b)}_i ,  \mu^{(b)}_j)$. The likelihood ratio statistic is
\begin{align*}
    LR^{(b)}_{i,j} = - 2 \log \frac{\max_{\mu^{(b)}_i =  \mu^{(b)}_j} L^{(b)}(\mu^{(b)}_i ,  \mu^{(b)}_j)}{\max_{\mu^{(b)}_i ,  \mu^{(b)}_j} L^{(b)}(\mu^{(b)}_i ,  \mu^{(b)}_j)}.
\end{align*}
The optimal values of the parameters under the unrestricted model are given by
\begin{align*}
\hat{\mu}_{C^{(b)}_i} &= \left( \sum_{n\in C^{(b)}_i} SD_n^{-2} \right)^{-1} \sum_{n\in C^{(b)}_i} \frac{\%\Delta_n}{SD_n^2},\\
\hat{\mu}_{C^{(b)}_j} &= \left( \sum_{m\in C^{(b)}_j} SD_m^{-2} \right)^{-1} \sum_{m\in C^{(b)}_j} \frac{\%\Delta_m}{SD_m^2}.
\end{align*}
Under the restricted model, we have
\begin{align*}
\hat{\mu}_{C^{(b)}_i} = \hat{\mu}_{C^{(b)}_j}  = \left( \sum_{n\in C^{(b)}_i \cup C^{(b)}_j} SD_n^{-2} \right)^{-1} \sum_{n\in C^{(b)}_i \cup C^{(b)}_j} \frac{\%\Delta_n}{SD_n^2}.
\end{align*}
We substitute these values and simplify to obtain
\begin{align*} LR^{(b)}_{i,j} = \left(S_{C^{(b)}_i} + S_{C^{(b)}_j}\right)^{-1} \left( \sqrt{\frac{S_{C^{(b)}_j}}{S_{C^{(b)}_i}}} \Tilde{\Delta}_{C^{(b)}_i} - \sqrt{\frac{S_{C^{(b)}_i}}{S_{C^{(b)}_j}}} \Tilde{\Delta}_{C^{(b)}_j}\right)^2,
\end{align*}
where we have, for a set $C$,
\begin{align}
    S_C &= \sum_{n\in C} SD_n^{-2},\label{eqn:summary-stats-precision} \\ 
    \Tilde{\Delta}_C &= \sum_{n\in C} \frac{\%\Delta_n}{SD_n^2}. \label{eqn:summary-stats-weighted-mean}
\end{align}
The following lemma is a straightforward consequence of the model used in this section along with central limit theorem.
\begin{lemma}
Under the null hypothesis $H_{i,j}^{(b)}$, $LR^{(b)}_{i,j}$ is a chi-square random variable on one degree of freedom. Therefore, $p_{i,j}^{(b)} = 1 - \chi_1^2(LR^{(b)}_{i,j})$ is a valid p-value for testing $H_{i,j}^{(b)}$.
\end{lemma}
These derivations are used in section \ref{sec:LikelihoodRatioComputation} to design an efficient algorithm for the use case of this section.

\subsection{Offline Evaluation of Fairness of Machine Learning Models}
Notions of fairness, which including equalized odds, equality of opportunity, individual or group fairness, and counterfactual fairness, are continually growing.  For brevity, we will discuss several metrics of fairness in the context of binary classification. Suppose that for each member $m$, covariates $X$ are used to predict a binary outcome $y$ through a model score, $s(X)$.  Suppose that a member belongs to one of $K$ (disjoint) groups $\mathcal{G}_1,...,\mathcal{G}_K$.  Suppose the the model classification of the observation, denoted by $c(X)$, is one if $s(X)$ exceeds a threshold $\tau$, and zero otherwise.  The classifier is said to achieved Equalized Odds if
\[P\left[c(X) = 1 | y, m \in \mathcal{G}_k \right] = P\left[c(X) = 1 |y, m \in \mathcal{G}_{k'} \right]\]
for all $k$, and $k'$.  The classifier satisfies equalized odds if 
\[P\left[c(X) = c |  m \in \mathcal{G}_k \right] = E\left[c(X) = c |m \in \mathcal{G}_{k'} \right]\]
for all $k$, and $k'$ and $c \in \{-1,1 \}$.
One can also assess parity with respect to any conventional metric assessing quality of a classifier (such as area under the receiver operator curve (AUC), True or False Positive Rate (TPR/FPR), Misclassification Rate, etc.) by requiring that the metric is equal across groups.  The vast majority of these metrics can be expressed as an appropriate sum, and inference resulting in a p-value is readily performed through a simple application of the central limit theorem. In particular, derivations of the previous section directly apply here.

\section{Efficient Computation}\label{sec:LikelihoodRatioComputation}
In this section we provide efficient implementation of the algorithm for the likelihood-ratio framework of previous sections. Algorithm \ref{algo:loglike} is designed based on derivations of section \ref{sec:notion-loglike} and is a recursive implementation of algorithm \ref{algo:clustering}. The computational complexity is $O(K^2)$.

\begin{algorithm}[!h]
	\caption{Detecting Heterogeneous Treatment Effects via Recursive Agglomerative Clustering}\label{algo:loglike}
	\begin{algorithmic}[1]
		\State \text{Input: } Groups $C_1,...,C_K$, metrics $(\% \Delta_1, SD_1),...,(\% \Delta_K, SD_K)$, and level of significance $\alpha$
		\State Output : A set of dissimilar clusters; in case of non-rejection only one cluster will be returned
		\State \text{Prepare data: } Clusters $\mathcal{C}_{\text{init}} = \{\{C_i\}\mid i = 1,\ldots, K\}$, summary metrics $\mathcal{D}_{\text{init}} = \{\Tilde{\Delta}_{C_i}\mid i = 1,\ldots, K \}$, $ \mathcal{S}_{\text{init}}=  \{S_{C_i} \mid i = 1,\ldots, K\}$ as defined in \ref{eqn:summary-stats-precision} and \ref{eqn:summary-stats-weighted-mean}, $\mathcal{P}_{\text{init}} = \{p_{i,j} \mid 1\leq i \le j \leq K\}$ where $p_{i,j}$ is defined through \ref{eqn:basicLRT}, level of significance $\alpha$, and the initial number of groups $K$
		\State \Return{\Call{TreatmentEffectsClustering}{$\mathcal{C}_{\text{init}}$, $\mathcal{D}_{\text{init}}$, $\mathcal{S}_{\text{init}}$, $\mathcal{P}_{\text{init}}$, $\alpha$, $K$ }}\\
		\Function{TreatmentEffectsClustering}{$\mathcal{C}$,\ $\mathcal{D}$,\ $\mathcal{S}$,\ $\mathcal{P}$, $\alpha$, $K$}
		    \State Set $p^\ast  = \max p_{i,j}\in \mathcal{P}$ and $\{i^\ast, j^\ast\} = \argmax p_{i,j}\in \mathcal{P}$.
		    \If{$p^\ast < \alpha/K^2$}
		        \State \Return $\mathcal{C}$
		    \Else
		        \State Define $C_{\text{new}} = C_{i^\ast} \cup C_{j^\ast}$ (the merged cluster), $ \Tilde{\Delta}_{C_{\text{new}}} =  \Tilde{\Delta}_{C_{i^\ast}} +  \Tilde{\Delta}_{C_{j^\ast}}$, and $S_{C_{\text{new}}} = S_{{C_{i^\ast}}} + S_{C_{j^\ast}}$, where  $S_{C}$ and $\Tilde{\Delta}_{C}$ are as defined in \ref{eqn:summary-stats-precision} and \ref{eqn:summary-stats-weighted-mean}
		        \State Define
		        \begin{align*}
		            \mathcal{C}_{\text{updated}} &= \text{concatenate}\left(\mathcal{C} / \{ C_{i^\ast} , C_{j^\ast}\}, C_{\text{new}}\right) \\
		            \mathcal{D}_{\text{updated}} &= \text{concatenate}\left(\mathcal{D} / \{ \Tilde{\Delta}_{C_{i^\ast}} , \Tilde{\Delta}_{C_{j^\ast}}\}, \Tilde{\Delta}_{C_{\text{new}}}\right)\\
		            \mathcal{S}_{\text{updated}} &= \text{concatenate}\left(\mathcal{S} / \{ S_{{C_{i^\ast}}} , S_{C_{j^\ast}}\}, S_{C_{\text{new}}}\right) 
		        \end{align*}
		        and $ \mathcal{P}_{\text{updated}}$ by removing dissimilarity scores of the pairs involving $\{i^\ast, j^\ast\}$, and adding those for $C_{\text{new}}$
		        \State \Return{\Call{TreatmentEffectsClustering}{$\mathcal{C}_{\text{updated}}$,\ $\mathcal{D}_{\text{updated}}$,\ $\mathcal{S}_{\text{updated}}$,\ $\mathcal{P}_{\text{updated}}$, $\alpha$, $K$}}
		    \EndIf
		\EndFunction
	\end{algorithmic}
\end{algorithm}

For a heuristic comparison of run-time complexity with the Causal Tree algorithm (which is effectively a CART algorithm), consider clustering based on a single categorical feature with $K$ possible outcomes.  Splitting the data based on this single feature requires $2^K$ comparisons, which is dramatically more computationally intensive than what is required by Algorithm \ref{algo:loglike}.

\section{Simulation Results}\label{sec:simulation}

In this section, we provide results based on simulated datasets which demonstrate the efficacy of our methodology. 

\subsection{Comparison with Causal Tree on a synthetic example}
Causal Tree, and other tree based methods, have proved to be very useful in
identifying heterogeneity of treatment effects, especially with respect to continuous covariates. However, they can have suboptimal performance on categorical features, particularly those with many levels, which are, of course, is of paramount importance in fairness context where the vast majority of covariates (e.g. gender, race, or country) are of this type.  On the contrary, our algorithm is naturally designed to work with categorical covariates. In order to understand the benefits of our methodology in these contexts, consider the following example. 

We consider countries in Asia and Africa.  From each country we randomly sample 100 people to be exposed to a control treatment and another 100 people to be exposed to a novel treatment. 
For each country, the outcome of the members assigned to the control are sampled as $N(0, 0.1)$.  For the countries in Asia, the outcome of the members assigned to the novel treatment are sampled as $N(- \mu, 0.1)$.  For the countries in Africa, the outcome of the members assigned to the novel treatment are sampled as $N(\mu, 0.1)$. 
In this context, there is heterogeneous treatment effects, with the continent membership as the source of heterogeneity.  

To compare the performance of Algorithm \ref{algo:loglike} with Causal Tree, we repeatedly simulate from this setting for values of $\mu$ taken between $0.0$ and $1$. We applied Algorithm \ref{algo:loglike} with $\alpha = 0.05$ and Causal Tree algorithm with the default settings \footnote{as given is the example usage at https://github.com/susanathey/causalTree}.
We use 20 equispaced points between zero and one for $\mu$. For each $\mu$, we repeatedly simulate from the above setting for 100 Monte Carlo samples, and use these results to understand how well the algorithms are performing the tasks of recovering the two heterogeneous continents as clusters of countries. Figure \ref{fig:power-curve} shows the probability of exact detection of Asia (the probability that one of the clusters of countries identified is exactly those countries in Asia) as a function of the treatment effect.

\begin{figure}[h]
  \centering
  \includegraphics[width=\linewidth]{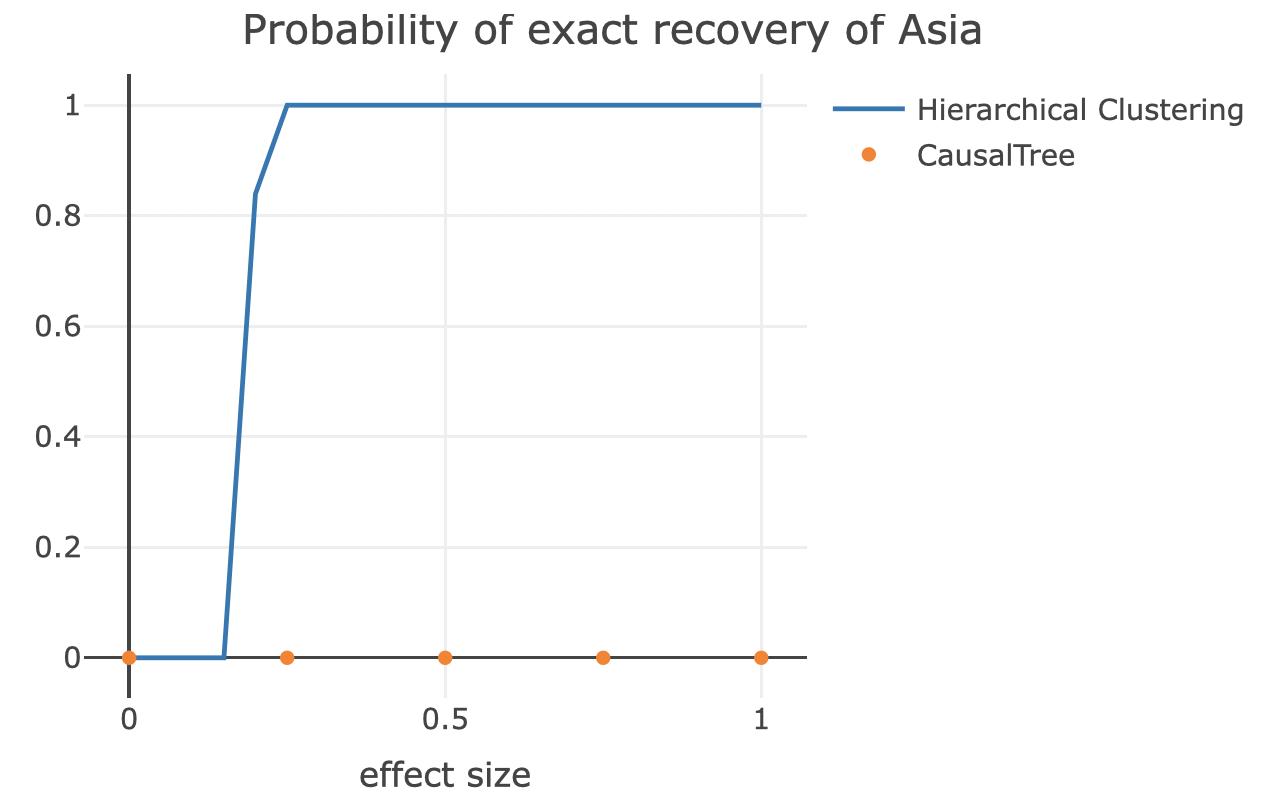}
  \caption{Power of the two algorithms for exact recovery of Asia.}\label{fig:power-curve}
\end{figure}

By $\mu = 0.2$, Algorithm \ref{algo:loglike} is reliably (with probability near one), able to recover the underlying structure of heterogeneity.  On the other hands, Causal Tree fails to accurately detect the cluster, even with $\mu$ up to $1$.  Furthermore, in addition to giving more accurate results than Causal Tree, our algorithm was substantially faster to run.  Causal Tree takes 95 seconds to run on this example, running R 3.4.1 on a MacBook Pro with a 3.5 GHz Dual-Core Intel Core i7 processor and 16GB of RAM. 
The likelihood-ratio hierarchical clustering algorithm solves the same problem in 27 milliseconds (on the same machine), offering more than three orders of magnitude speed-up. In fact, the computational cost of Causal Tree prohibits larger scale simulations, which is why we limited the study to countries in Asia and African rather than all countries.

While \cite{athey16} suggests that causal tree can be used to identify heterogeneous subpopulations (and it has been used to this effect e.g. for personalized health care by \cite{Wang16}), the primary goal of causal tree is not to identify clusters accurately, but rather to accurately predict treatment effects.  We now compare the accuracy of treatment effect predictions using our method with those given by causal tree.

When the treatment effect is zero, there is no heterogeneity and both methods effectively use all available data to predict the mean treatment effect, resulting in very accurate predictions with the hierarchical clustering method performing somewhat better. However, this changes as the treatment effect increases. For small values of the treatment effect, the hierarchical clustering algorithm performs worse than the causal tree. On the other hand, when the treatment effect is large enough that the heterogeneity is detectable, the causal tree methodology is unable to accurately capture the structure of the heterogeneity, resulting in poor treatment effect estimates.  By contrast, our method more accurately identifies the patterns in the data, allowing more data to be leveraged for treatment effect estimates, leading to much better accuracy. In terms of relative mean square error efficiency, CausalTree is at most 15 times more efficient for smaller values of $\mu$ but can be 80 times less efficient for other values.

\begin{figure}[h]
  \centering
  \includegraphics[width=\linewidth]{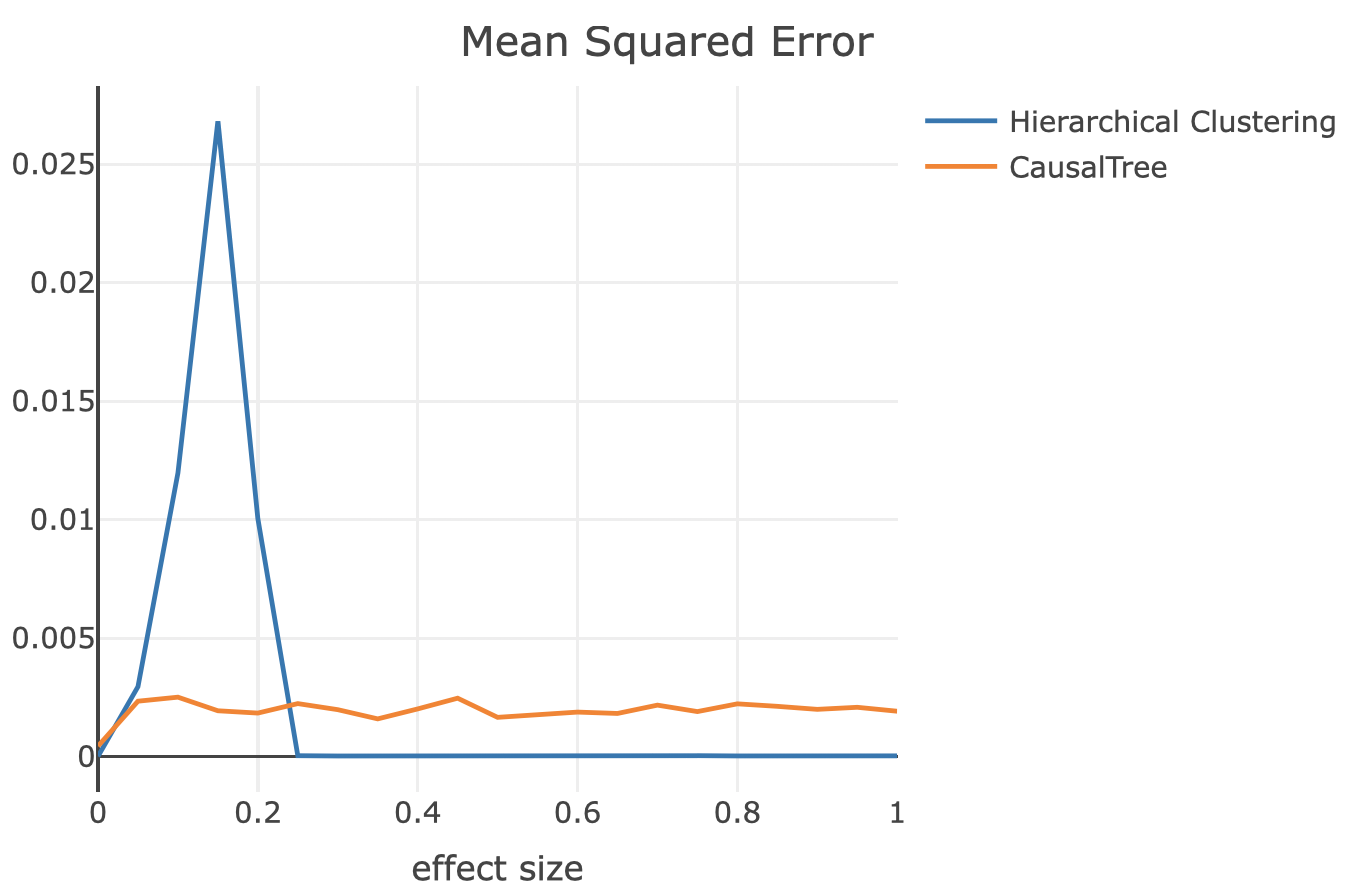}
  \caption{Mean squared error of the two algorithm for estimation of the treatment effect in a specific country.}\label{fig:mse}
\end{figure}

\section{Experimental Results}
\subsection{False positive rate under the null}
We study the performance of our algorithm, using the Bonferroni type correction, on real data from LinkedIn's experimentation platform. We consider the problem of detecting heterogeneity of treatment effect across different countries.  First, focus on the null case. We looked at 2218 samples from dummy experiments running during January 2020. These were experiments which served the same variant, which could vary from one experiment to another, to both the control and treatment group. Consequently, this provides the null setting for for problem where the treatment effect zero in all countries. In this specific setting there are 21 subgroups that are being clustered and the figure \ref{fig:level-curve} shows the false rejection rate for different values of $\alpha$ under the null hypothesis. The test in slightly anti-conservative for small values of $\alpha$ and somewhat conservative for larger values $\alpha$ as expected from a Bonferroni type correction.
\begin{figure}[H]
  \centering
  \includegraphics[width=\linewidth]{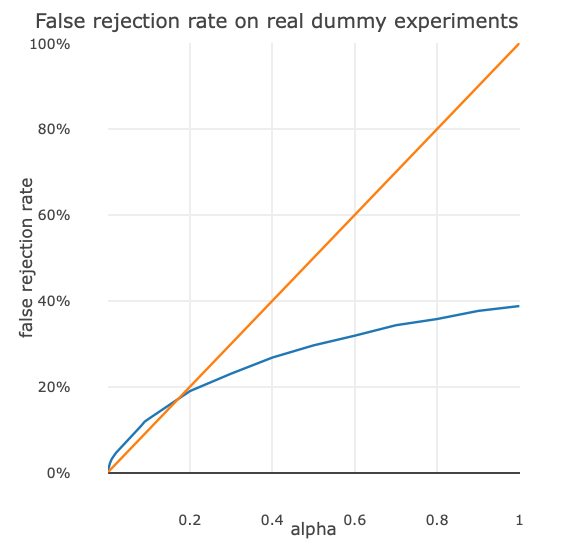}
  \caption{False positive error rate of the algorithm as a function of $\alpha$.}\label{fig:level-curve}
\end{figure}

In practice, one would have to choose a value of $\alpha$. Choice of $\alpha$ here should be done in a similar manner to standard A/B tests. $\alpha$ represents the level of error that the practitioner is willing to tolerate in exchange for a certain precision in a specific application. While the most common choice is $\alpha = 0.05$, this is subjective and should be left to the practitioner who can rely on their domain expertise and weight the trade-offs inherent to their specific application.

\subsection{Comparison with a clustering algorithm based on Gaussian mixture models}\label{sec:experiment}
Gaussian mixture models \cite{reynolds2009gaussian} (GMM) are one of the most widely used clustering algorithms across variety of applications. GMMs can be utilised as a method to detect heterogeneous treatment effects. In fact, a Gaussian mixture model is implemented in Microsoft's experimentation system as a tool for detection of heterogeneous treatment effects \cite{alexDeng}. This algorithm assumes fixed variances for each data point given by the estimated variance from the data and only estimates the means through a customized EM algorithm. Number of clusters is an input parameter to these types of algorithms and choosing the number of clusters is usually a great challenge. The GMM used in Microsoft's experimentation system uses BIC and AIC. One of the advantages of our clustering algorithm is that the number of clusters are determined by the algorithm itself and the user only has to specify the false positive tolerance $\alpha$.

We applied this method to the dummy experiments and compared the results with our algorithm. The GMM rejected the null hypothesis 7 times out or 2218 samples. This is similar to $\alpha = 0.001$ in our algorithm. Out of those seven, four of them were also detected by our algorithm for $\alpha < 0.001$. Figure \ref{fig:dummy-comparison} shows the resulting clusterings from the two algorithms in one of these cases; the two methods gave similar results for three out of four common cases and we only present the case where they differed significantly. As can be seen, the clusters generated by our algorithm are more intuitive as the means for data points are clearly separated by a single cutoff, whereas the clusters generated by the GMM are somewhat arbitrary. This is likely because the GMM is indifferent to class assignment of point with higher variance as then can be assigned to different components of the mixture models.
\begin{figure}
  \centering
  \includegraphics[width=\linewidth]{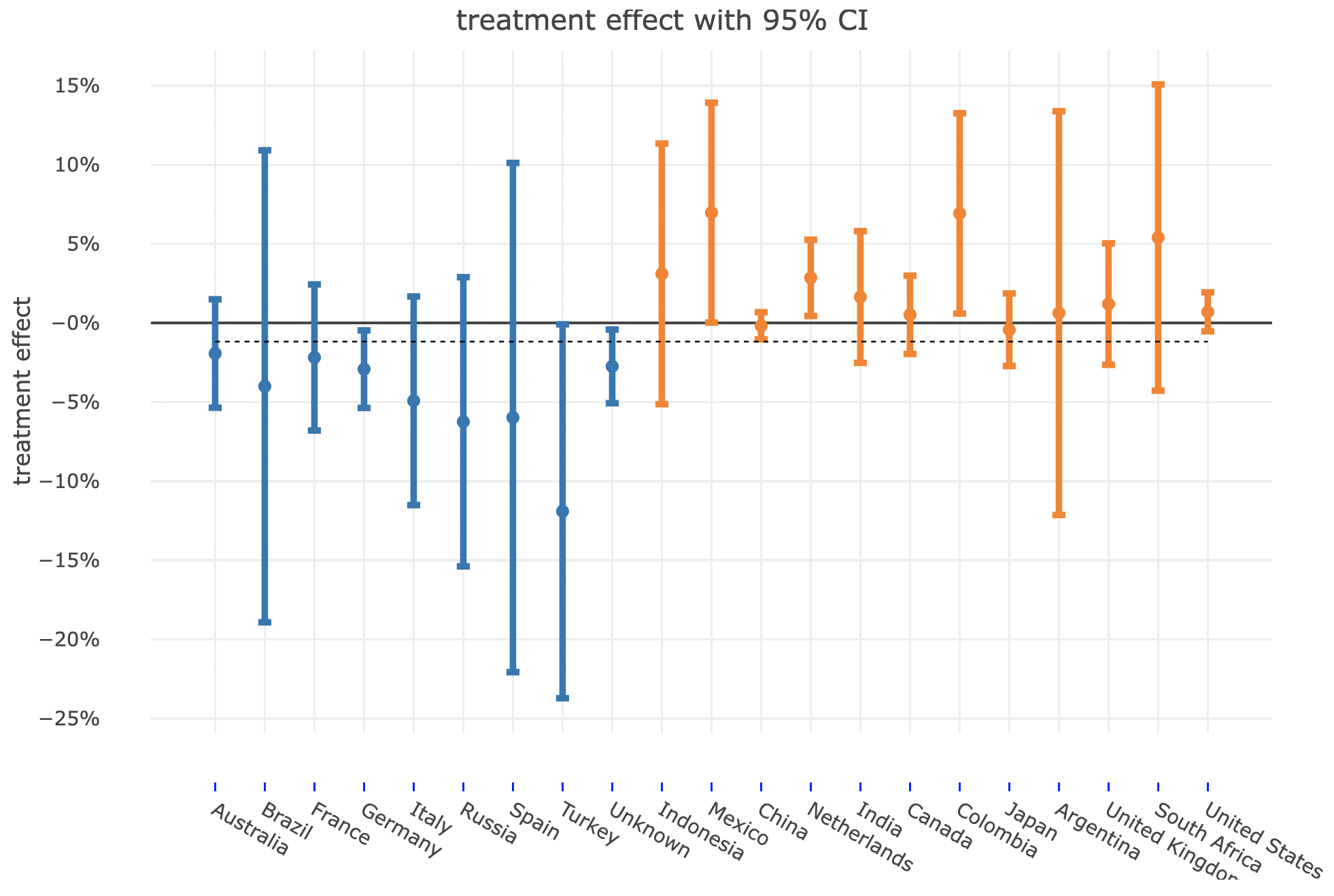}\\
  \includegraphics[width=\linewidth]{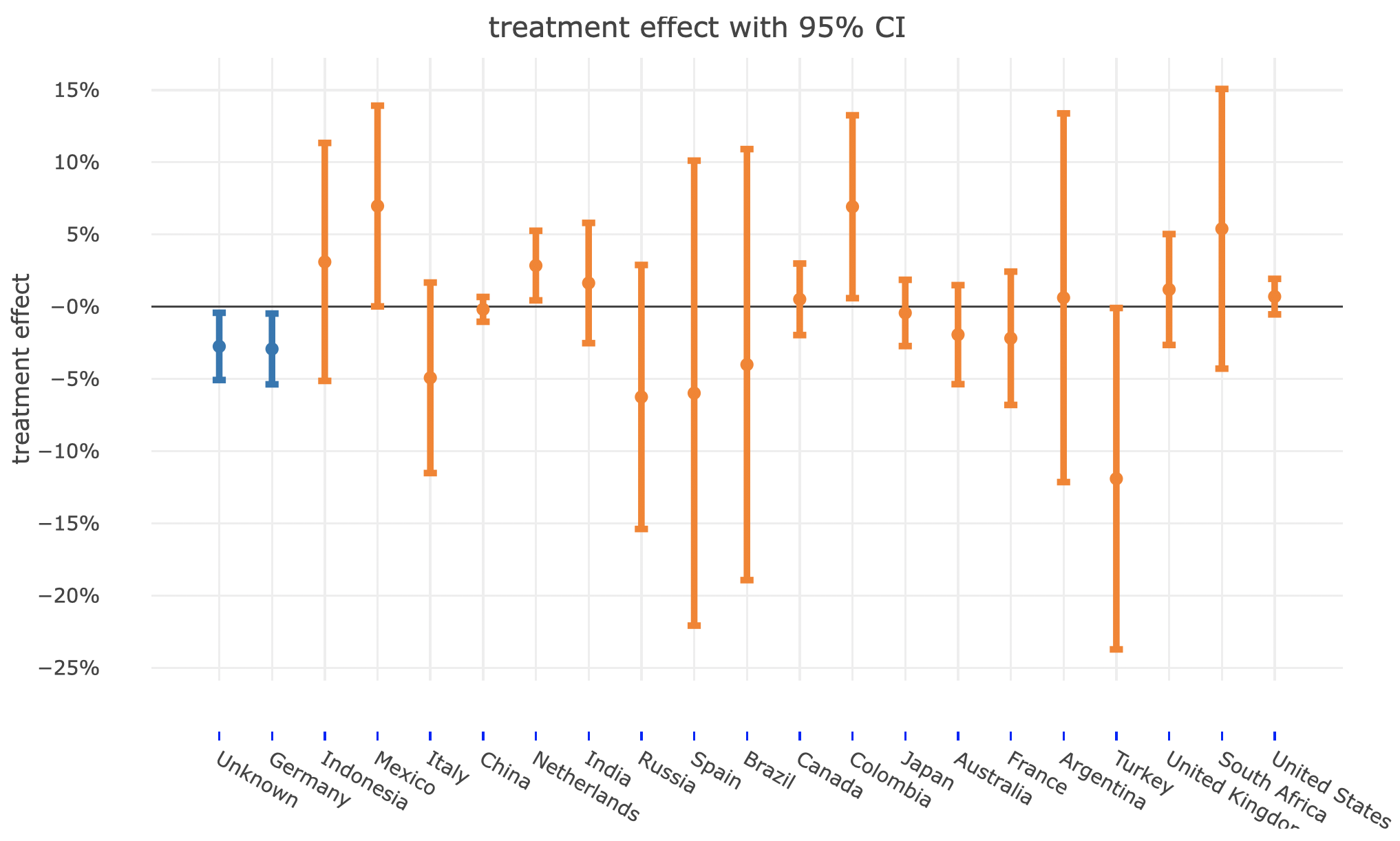}
  \caption{Clusters generated by: likelihood-ratio hierarchical clustering (top) and Gaussian mixtures models (bottom). Each color corresponds to a cluster.}\label{fig:dummy-comparison}
\end{figure}

\subsection{Example of an experiment: new user on-boarding}
We look at an experiment that demonstrated heterogeneity in treatment effects. The experiment A/B test two different on-boarding experiences for new users. The metric we consider is the percentage of user who had a session after opening an email we sent them. This measures how likely new users are to interact with the on-boarding process. Figure \ref{fig:real-comparison} illustrates the clusters generated by the two algorithms discussed in the previous section. Again, the clusters generated by our algorithm are more interpretable as they are given by two cutoffs for the treatment effects, whereas, the clusters from the GMM include overlapping intervals. 

\begin{figure}[h]
  \centering
  \includegraphics[width=\linewidth]{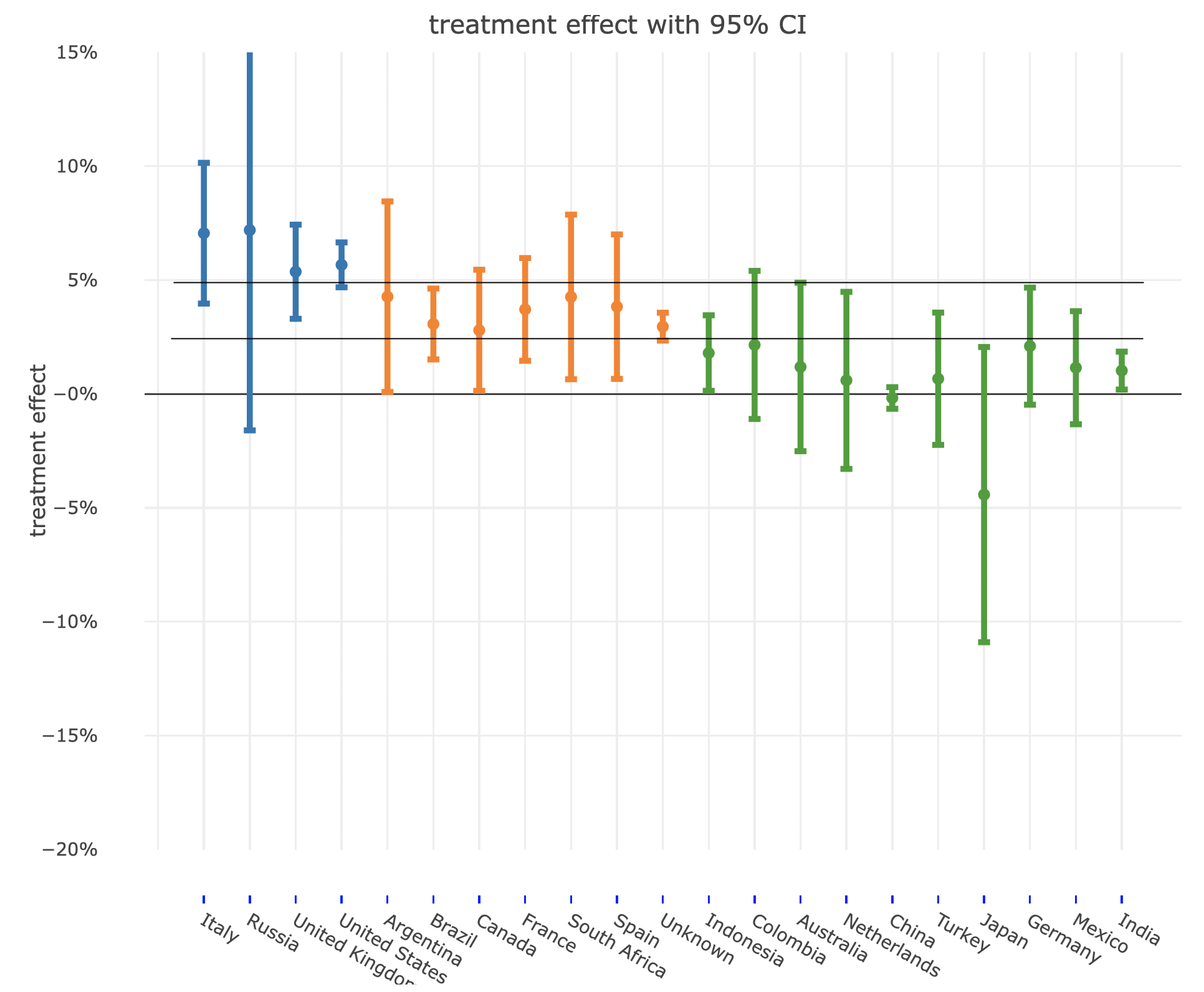}\\
  \includegraphics[width=\linewidth]{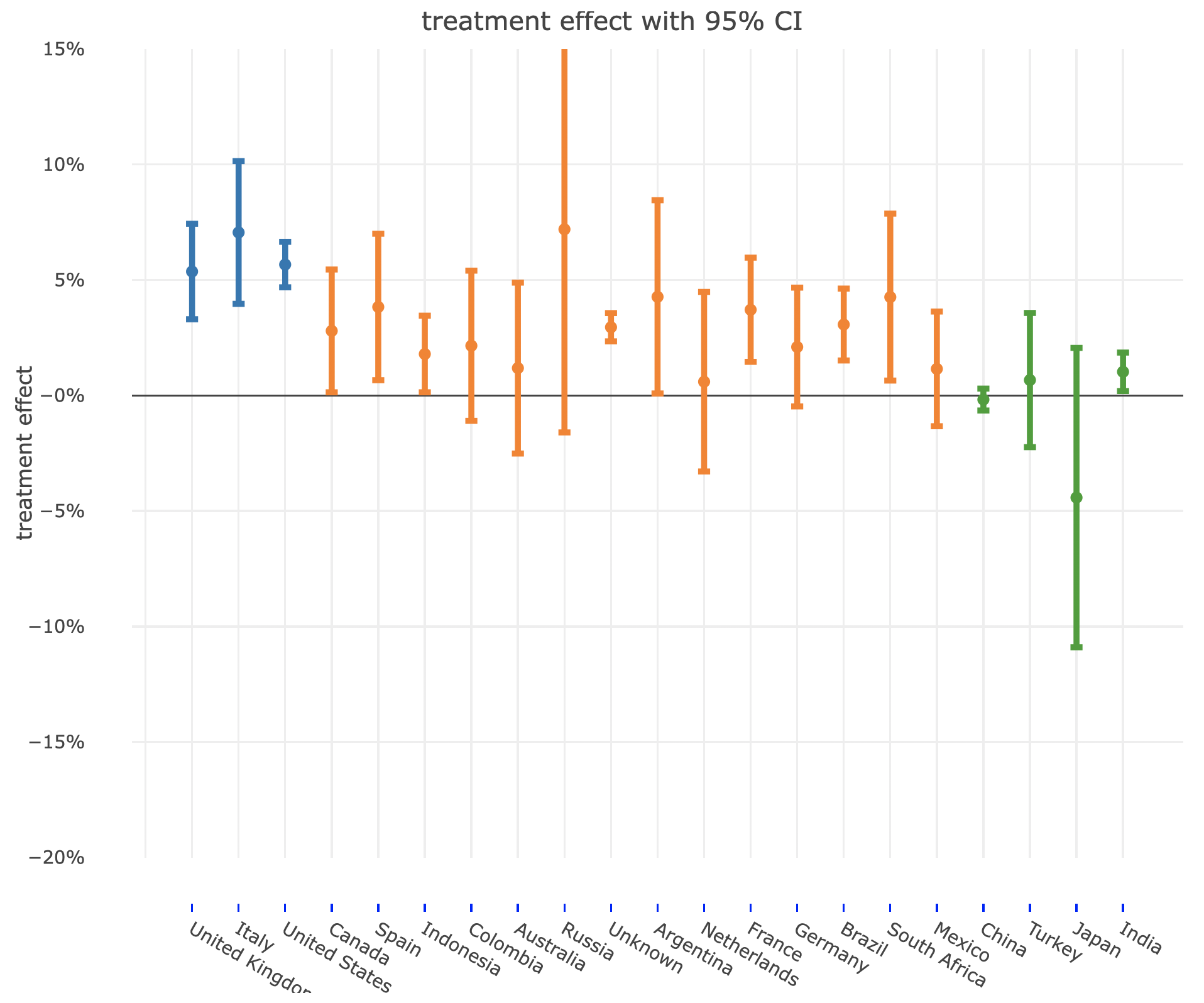}
  \caption{Clusters generated by: likelihood-ratio hierarchical clustering (top) and Gaussian mixtures models (bottom). Each color corresponds to a cluster.}\label{fig:real-comparison}
\end{figure}

\section{Conclusion}

We provide a hierarchical clustering algorithm which relies on a notion of group similarity which can be established through a rigorous statistical test.  We develop a novel choice of test statistics for evaluating equality of lift, and recommend choices of well known fairness metrics.    
This method is applicable to a holistic approach towards model or product evaluation, in which fairness can be evaluated offline, and business impact can be evaluated online.  In these contexts, our algorithm provides interpretable groupings of members for whom the product demonstrates comparable performance.

\section{Appendix}

\subsection{Reproducibility}
A link to a git repository containing the relevant code for the simulations will be included in this section.  This is currently omitted to maintain the double blind review.

\subsection{Proofs}
\begin{proof}[Proof of Theorem \ref{thm:validity}]
Note that if the hypothesis
\begin{equation}\label{eqn:GlobalNull}
H_0: C_i \text{ is similar to } C_j \text{ for all } i \neq j, 1 \leq i,j \leq K.     
\end{equation}
is true, then we also have that the null hypothesis 
\[
H_0^{(b)}: \text{ There exists a pair of clusters } C_i \text{ and } C_j \text{ such that }  C_i \text{ is similar to } C_j
\]
is also true because of the merge invariance under $H_0$.  

Now, if $H^{(b)}_0$ is true, then there exist some indices $i^*$ and $j^*$ such that 
\[
H_{i^*,j^*}^{(b)}: C^{(b)}_i \text{ is similar to } C^{(b)}_j
\]
is true.  Consequently, 
\begin{align*}
    P(\text{Reject } H^{(b)}_0) &= P(\text{Reject } H_{i,j}^{(b)} \text{ for all } i \text{ and } j )\\
    &\leq P(\text{Reject } H_{i^*,j^*}^{(b)})\\
    & \leq \alpha/K~.
\end{align*}
Therefore, the result follows from the Bonferroni correction and the fact that we are sequentially testing at most $K$ such hypotheses $H_0^{(b)}$.
\end{proof}

\subsection{On the merge invariance criteria}

The merge invariance criteria is needed to maintain that merging clusters preserves similarity when the original clusters are similar.  This is readily verified in many contexts.

Consider the case of lift, and suppose three clusters have means $\mu_i$, $i=1,2,3$ under a control model, and $l \cdot \mu_i$ under a treatment so that the lift for each cluster is $l$.  In this case, if the proportion of members in each group is $p_i$, then the lift of the cluster formed by merging the first two clusters is 
\[
\frac{p_1 l \cdot \mu_i + p_2 l \cdot \mu_2}{p_1 \mu_i + p_2 \mu_2} = l~.
\]
Therefore, the merged cluster is similar to the third cluster.  

To provide an example in the fairness context, suppose we say that two clusters are similar if their false positive rate differs by no more than a specified tolerance $\epsilon$ (e.g. the FPR is within 1\%).  Then because the false positive rate of a cluster formed by merging two similar clusters is a weighted average which must lie between the false positive rates of the two clusters, it is easily seen that the merge invariance property holds. 

To demonstrate a scenario where this may fail, consider defining clusters as similar if they have equal variance.  If three clusters have equal variance, merging the first two may result in a group with different variance than the original common variance.  For this reason, the cluster formed by merging the first two clusters will no longer be similar to the third cluster.

\begin{acks}
The authors would like to thank Alex Deng for sharing his work on heterogeneous treatment effects at Microsoft which was used for comparison in our simulations. We would also like to thank Nanyu Chen, Weitao Duan, Zhoutong Fu, Justin Marsh, and Praveen Gujar for helpful discussions.
\end{acks}

\bibliographystyle{ACM-Reference-Format}
\bibliography{sample-base}

\appendix

\end{document}